\documentclass[dvipsnames,format=sigconf]{acmart}

\AtBeginDocument{%
  \providecommand\BibTeX{{%
    \normalfont B\kern-0.5em{\scshape i\kern-0.25em b}\kern-0.8em\TeX}}}

\copyrightyear{2023}
\acmYear{2023}
\setcopyright{acmlicensed}
\acmConference[GECCO '23]{Genetic and Evolutionary Computation Conference}{July 15--19, 2023}{Lisbon, Portugal}
\acmBooktitle{Genetic and Evolutionary Computation Conference (GECCO '23), July 15--19, 2023, Lisbon, Portugal}
\acmPrice{15.00}
\acmDOI{10.1145/3583131.3590375}
\acmISBN{979-8-4007-0119-1/23/07}

\usepackage[ruled]{algorithm2e}
\usepackage{multirow}
\usepackage{booktabs}
\usepackage{graphicx}

\usepackage{xspace}

\newcommand{\eg}{{\em e.g.}\xspace}
\newcommand{\ie}{{\em i.e.}\xspace}

\begin{document}

\title{Probabilistic Lexicase Selection}

\author{Li Ding}
\affiliation{%
  \institution{University of Massachusetts Amherst}
  \city{Amherst}
  \state{MA}
  \country{}
}
\email{liding@umass.edu}

\author{Edward Pantridge}
\affiliation{
  \institution{Swoop}
  \city{Cambridge}
  \state{MA}
  \country{}
}
\email{ed@swoop.com}

\author{Lee Spector}
\affiliation{%
  \institution{Amherst College}
  \institution{University of Massachusetts Amherst}
  \city{Amherst}
  \state{MA}
  \country{}
}
\email{lspector@amherst.edu}

\renewcommand{\shortauthors}{Ding, et al.}


\begin{abstract}
  Lexicase selection is a widely used parent selection algorithm in genetic programming, known for its success in various task domains such as program synthesis, symbolic regression, and machine learning. Due to its non-parametric and recursive nature, calculating the probability of each individual being selected by lexicase selection has been proven to be an NP-hard problem, which discourages deeper theoretical understanding and practical improvements to the algorithm. In this work, we introduce probabilistic lexicase selection (plexicase selection), a novel parent selection algorithm that efficiently approximates the probability distribution of lexicase selection. Our method not only demonstrates superior problem-solving capabilities as a semantic-aware selection method, but also benefits from having a probabilistic representation of the selection process for enhanced efficiency and flexibility. Experiments are conducted in two prevalent domains in genetic programming: program synthesis and symbolic regression, using standard benchmarks including PSB and SRBench. The empirical results show that plexicase selection achieves state-of-the-art problem-solving performance that is competitive to the lexicase selection, and significantly outperforms lexicase selection in computation efficiency.
\end{abstract}



\maketitle

\section{Introduction}

Parent selection is an essential component in genetic and evolutionary algorithms, which determines a set of individuals to use as the source to create off-sprints for the next generation. Among many parametric and non-parametric selection methods, lexicase selection~\citep{helmuth2014solving,spector2012assessment} has shown state-of-the-art performance in genetic programming, and has been successfully extended to other problem domains such as symbolic regression~\cite{la2016epsilon,la2018learning}, machine learning~\cite{la2020genetic}, and deep learning~\cite{ding2022optimizing}.

As a semantic selection method, the lexicase selection algorithm\footnote[1]{For comprehensive insights into the lexicase selection algorithm, detailed descriptions can be found in Appendix Sec.~\ref{asec:lex}.} evaluates individuals on each single training case in random orders. One major drawback of lexicase selection is that the selection process remains a black-box due to its non-parametric and recursive nature. As a result, numerous repeated selection events are usually required when using lexicase selection, as opposed to more efficient probabilistic methods, \eg, fitness proportionate selection~\cite{golberg1989genetic,lipowski2012roulette}, that directly sampling the individuals from a probability distribution calculated from fitness measures.

While repetitions of lexicase selection events eventually produce a probability distribution of individuals to be selected, the underlying recursion introduces step-wise dependency and thus prevents a straightforward way to numerically calculate the probability distribution of individuals to be selected. Nonetheless, there has been rising interest in analyzing lexicase selection from a probabilistic point of view~\cite{la2019probabilistic}, which aims to develop a deeper theoretical understanding of lexicase selection. In fact, recent work~\cite{dolson2023calculating} has proved that the exact calculation of the probabilities of individuals being selected in lexicase selection is nevertheless an NP-Hard problem, which implies the potential need for an approximation solution that would help develop better theoretical interpretation and practical improvement to lexicase selection.

With the above motivations in mind, we propose probabilistic lexicase selection (plexicase selection), which is, to our knowledge, the first semantic-aware selection method that has a probabilistic representation of selection events. Our method efficiently calculates an approximation to the probability distribution of individuals to be selected by lexicase selection, and samples individuals from this distribution instead of actually performing selection. There are two main advantages of the proposed method. Firstly, plexicase selection directly calculates the probability of individuals being selected instead of performing repeated selection events, which can reduce the algorithm runtime significantly. Secondly, having the probability distribution of individuals allow us to perform parametric optimization on the selection process. We introduce a probability manipulation process to plexicase selection, which uses a hyperparameter to control the kurtosis of the distribution of individuals, and thus manages to further enhance the problem-solving performance.

In order to illustrate the advantages of plexicase selection, we conduct experiments in two prevalent domains: program synthesis and symbolic regression. Results on standard benchmarks including the General Program Synthesis Benchmark Suite (PSB)~\cite{helmuth2015general} and Symbolic Regression Benchmark (SRBench)~\cite{la2021contemporary} suggest that plexicase selection significantly outperforms lexicase selection in terms of efficiency, and at the same time demonstrates superior or competitive problem-solving capabilities.

The paper is organized as follows. In Section~\ref{sec:2}, we summarize the background of lexicase selection and the challenges that motivate our work. Section~\ref{sec:3} explains the theories and methodologies of the proposed algorithm in detail, and Section~\ref{sec:4} describes the experiments and results in two domains where genetic programming has been extensively adopted. The results include various measures of problem-solving performance, algorithm runtime, and ablation studies on comparisons and hyperparameters. We conclude with a discussion of limitations and future work.

\section{Related Work}
\label{sec:2}

Lexicase selection~\cite{helmuth2014solving,spector2012assessment} assesses performance on individual training cases instead of employing aggregated accuracy or fitness metrics. It is a semantic-aware selection method~\cite{liskowski2015comparison} with the benefit of including semantic information regarding the population. In other words, lexicase selection tends to select specialist individuals for being elite on a subset of cases~\cite{helmuth2019lexicase,helmuth2020importance}, which enhances population diversity~\cite{helmuth2016lexicase,moore2018tiebreaks}. Lexicase selection has been demonstrated more reliable performance than other contemporary genetic programming methods~\cite{fieldsend2015strength,galvan2013using,krawiec2015automatic}. Recent work has empirically studied and extended lexicase selection to other supervised and unsupervised learning domains, \eg, symbolic regression~\citep{la2016epsilon}, machine learning~\citep{la2020learning,la2020genetic}, rule-based learning~\citep{aenugu2019lexicase}, deep learning~\citep{ding2021evolving,ding2022optimizing}, and evolutionary robotics~\citep{huizinga2018evolving,la2018behavioral}, which illustrates its capability of improving model performance and generalization.

Lexicase selection involves an essential procedure of gradually eliminating individuals by evaluating the population on each case. Recent work started to investigate the probability distribution of individuals selected by lexicase selection. \citet{helmuth2016impact} explored the correlation between the probability of selecting an individual and its rank in a population. \citet{la2019probabilistic} derived the expected probabilities of lexicase selection and studied the effect of varying population and training set sizes. \citet{dolson2023calculating} further proved that the problem of calculating the exact probabilities of lexicase selection is NP-Hard, which inspired this work to alternatively design an approximation algorithm to obtain the probability distribution.

Probability-based selection methods have been studied for decades in genetic and evolutionary computation. The basic strategy is to assign higher selection probabilities to better-fitted individuals based on their aggregated fitness. The most common method in the category is fitness proportionate selection~\cite{golberg1989genetic,lipowski2012roulette}, also known as roulette wheel selection, which assumes the probability of selection is proportional to the fitness of an individual. Extensions~\cite{baker2014adaptive} to fitness proportionate explored using the ranking or ordering of individuals rather than their raw fitness. There have been some other methods with selection probabilities depending on fitness, \eg, remainder stochastic independent sampling~\cite{baker1987reducing} and stochastic universal selection~\cite{goldberg1991comparative}. While these methods have slightly different statistical properties, none of them takes into account the program semantics~\cite{liskowski2015comparison}, \ie, the vector of outputs returned by a given program for a given set of fitness evaluations. To our knowledge, the proposed plexicase selection is the first semantic-aware selection method that has a probabilistic representation of the population, which has substantial advantages in both problem-solving performance and efficiency, as described in detail in Section~\ref{sec:pls}.

From the perspective of algorithm efficiency, there have been a few attempts to reduce the runtime of lexicase selection. \citet{helmuth2022population} theoretically analyzed the runtime of lexicase selection and found that the expected runtime depended on population diversity. \citet{de2019batch} introduced a hybrid selection method combining the idea of tournament and lexicase selection to improve both efficiency and quality of solutions. \citet{aenugu2019lexicase} introduces a batched variant of lexicase selection, which assesses batched data samples instead of individual cases during selection events. Most recently, \citet{ding2022going,ding2022lexicase} proposed a partial evaluation method based on weighted shuffle~\cite{troise2018lexicase} to reduce the total number of evaluations. In general, most of the prior work focuses on optimizing the single selection event, but still requires repetitions of selection events to complete parent selection of the whole population. Our work takes a different direction: first, we calculate the probability distribution of individuals to be selected, and then we perform all the selection events at once by sampling from the distribution. Such an approach produces significant improvement in efficiency for up to $10$ times speed-up, as shown in our empirical results on the program synthesis task.

Another popular trend in developing practical improvements to lexicase selection is to utilize the downsampling of training cases. \citet{hernandez2019random} proposed downsampled lexicase selection and \cite{ferguson2019characterizing} studied it further. Recent work~\cite{helmuth2020explaining,helmuth2022problem} investigated the reasons behind the effectiveness of downsampling, and proposed a few hypotheses. \citet{moore2017lexicase} also uses downsampling with lexicase selection, but with different terminology. \citet{boldi2023informed} recently introduce informed downsampled lexicase selection, which leverages population statistics to build downsamples that contain more distinct and informative training cases. In this work, we demonstrate that the proposed plexicase selection method works well with existing downsampling terminologies, and the benefits of adding downsampling to plexicase selection are more significant compared to lexicase selection.

\section{Methods}
\label{sec:3}

In this section, we describe the proposed method in detail. The plexicase selection method has two main components: finding Pareto set boundaries and assigning probabilities to individuals. We start by reviewing the preliminaries on Pareto set boundaries and their connection with lexicase selection, followed by proposing an algorithm to efficiently obtain Pareto set boundaries through pairwise comparisons. We then introduce a method to assign probabilities to individuals, and finally present plexicase selection and its extension with $\epsilon$-relaxation.

\subsection{Preliminaries on Pareto Set Boundaries}
\label{sec:3.1}

The behavior of lexicase selection can be alternative interpreted as multi-objective optimization with respect to the training cases, where each training case is considered as a stand-alone objective. \citet{la2019probabilistic} has proved that lexicase selection only selects \textit{Pareto set boundaries}, \ie, individuals that are in the boundaries of the Pareto front. We first review the definitions related to Pareto set boundaries in the context of evolutionary computing.

Given a population $\mathcal{Y}$, $f_j(y_i)$ denotes the fitness (objective) function for the $i$th individual $y_i$ on the $j$th case $x_j$ in the training set $\mathbb{X}$ of size $N$. In this work, we assume the goal of optimization is to maximize the fitness function.

\begin{definition}
  \label{def:dominate}
  For individuals $y_1, y_2 \in \mathcal{Y}$,  if $f_j(y_1) \geq f_j(y_2) \ \forall j \in \{1,\cdots,N\}$, we state $y_1 \succeq y_2$ ($y_1$ dominates $y_2$).
\end{definition}

\begin{definition}
  \label{def:strictdomination}
  If $y_1 \succeq y_2$ and $\exists j\in\{1,\cdots,N\}$ for which $f_j(y_1) > f_j(y_2)$, we state $y_1 \succ y_2$ ($y_1$ strictly dominates $y_2$).
\end{definition}

In parent selection, if two individuals have the exact same fitness on every training case, the selection algorithm can not distinguish them. In this work, we only consider the situation where the individuals have distinct overall fitness evaluations, \ie, $\exists j\in\{1,\cdots,N\}$ for which $f_j(y_i) \neq f_j(y_k) \ \forall y_i, y_k \in \mathcal{Y}$. We show that the domination in this case can be replaced by its strict form.

\begin{lemma}
  \label{lem:strict}
  If an individual is dominated by another individual, and they have distinct case-wise fitness, the domination is strict.
\end{lemma}

\begin{proof}
  Let $y_1 \succeq y_2$, we have $f_j(y_1) \geq f_j(y_2) \ \forall j \in \{1,\cdots,N\}$. If $y_1$ and $y_2$ have distinct case-wise fitness, we have $\exists k\in\{1,\cdots,N\}$ for which $f_k(y_1) \neq f_k(y_2)$. Thus, we have $f_k(y_1) > f_k(y_2)$. According to Def.~\ref{def:strictdomination}, we can obtain $y_1 \succ y_2$.
\end{proof}

The Pareto set $\mathcal{P}$ of $\mathcal{N}$ is a subset ($\mathcal{P} \subseteq \mathcal{N}$) where each individual is non-dominated. We have the following definitions regarding Pareto set with respect to parent selection.

\begin{definition}
  \label{def:pareto}
  Given $y_i \in \mathcal{N}$, if $y_i \nprec y_k \ \forall y_k \in \mathcal{Y}$, $y_i \in \mathcal{P}$ where $\mathcal{P}$ is the Pareto set.
\end{definition}

\begin{definition}
  \label{def:paretobound}
  Given $y_i \in \mathcal{P}$, if $\exists j\in\{1,\cdots,N\}$ for which $f_j(y_i) \geq \max_{y_k\in \mathcal{Y}} f_j(y_k)$, $y_i$ is a Pareto set \textit{boundary}.\footnote[2]{We follow the definition of Pareto set \textit{boundary} in \citet{la2019probabilistic}. Note that the Pareto set boundary in this work refers to the `boundary' of the Pareto set, which is different from Pareto set (or Pareto boundary used in other work).}
\end{definition}

\citet{la2019probabilistic} proved the following, which appears as Theorem 3.4 in their paper:

\begin{theorem}
  \label{thm:pareto}
  If individuals from a population $\mathcal{Y}$ are selected by lexicase selection, those individuals are Pareto set boundaries of $\mathcal{Y}$ with respect to the training set $\mathbb{X}$.
\end{theorem}

\subsection{Obtaining Pareto Set Boundaries through Pairwise Comparisons}
\label{sec:pairwise}

Suppose we have the global best fitness of $x_j$ over population $\mathcal{Y}$ as $f^*_j(\mathcal{Y}) = \max_{y_i\in \mathcal{Y}} f_j(y_i)$, we show that an individual must achieve global best fitness on at least one case in order to be selected by lexicase selection.

\begin{lemma}
  \label{lem:1}
  If an individual $y_i \in \mathcal{Y}$ is selected by lexicase selection, then $\exists j\in\{1,\cdots,N\}$ for which $f_j(y_i) = f^*_j(\mathcal{Y})$.
\end{lemma}

\begin{proof}
  According to Theorem~\ref{thm:pareto}, $y_i$ must be a Pareto set boundary, \ie, $\exists j\in\{1,\cdots,N\}$ for which $f_j(y_i) \geq \max_{y_k\in \mathcal{Y}} = f^*_j(\mathcal{Y})$.
\end{proof}

Besides, we show that the number of elitism $E(y_i)=|\{f_j(y_i) = f^*_j(\mathcal{Y}) \ \forall j \in \{1,\cdots,N\}\}|$, \ie, the number of cases that an individual achieves global best fitness on, can be useful in determining dominance.

\begin{theorem}
  \label{thm:elitism}
  For two individuals $y_1, y_2 \in \mathcal{Y}$, if $E(y_1)>E(y_2)$, $y_1 \nprec y_2$.
\end{theorem}

\begin{proof}
  Given that $E(y_1)>E(y_2)$, there must $\exists j\in\{1,\cdots,N\}$ for which $f_j(y_1) = f^*_j(\mathcal{Y}) > f_j(y_2)$. According to Definition~\ref{def:dominate}, we have $y_1 \nprec y_2$.
\end{proof}

\begin{algorithm}[t]
  \KwData{
    \begin{itemize}
      \item $f(y_i)$: the fitness vector of individuals $y_i \in \mathcal{Y}$
      \item $\mathcal{Y}$: the population of individuals
    \end{itemize}
  }
  \KwResult{
    \begin{itemize}
      \item $\mathcal{P}$: a set of individuals that are Pareto set boundaries
    \end{itemize}
  }

  $f^*(\mathcal{Y}) \gets \max_{i}f(y_i)$

  $\mathcal{P} \gets$ $\mathcal{Y}$ sorted by number of elitism $E(y_i)$

  \For{$y_i$ in $\mathcal{P}$}{
    \If{$E(y_i)==0$}{
      Remove $y_i$ from $\mathcal{P}$
    }\Else{
      \texttt{individuals\_to\_compare} $\gets \{y_k\} \ \forall E(y_k) \leq E(y_i)$.
      \For{$y_k$ in \texttt{individuals\_to\_compare}}{
        \If{$\sum{f(y_i) < f(y_k)}=0$}{
          Remove $y_k$ from $\mathcal{P}$
        }
      }
    }
  }

  \KwRet{$\mathcal{P}$}

  \caption{Find Pareto Set Boundaries through
    Pairwise Comparisons}
  \label{alg:pareto}
\end{algorithm}

We can also extend Theorem~\ref{thm:pareto} to the case of pairwise comparison of two individuals.

\begin{lemma}
  \label{lem:2}
  Given two individuals $y_1, y_2 \in \mathcal{Y}$, $y_2$ can not be selected by lexicase selection if $y_1 \succ y_2$.
\end{lemma}

\begin{proof}
  If $y_1 \succ y_2$, $y_2$ is not in the Pareto set according to Definition~\ref{def:pareto}. So $y_2$ can not be selected by lexicase selection according to Theorem~\ref{thm:pareto}.
\end{proof}

Combining Lemma~\ref{lem:1} and Lemma~\ref{lem:2}, we show a way of determining whether an individual can be selected by lexicase selection through global best fitness and pairwise comparisons.

\begin{theorem}
  \label{thm:lexiprob}
  An individual $y_i \in \mathcal{Y}$ is a Pareto set boundary, \ie, can be selected by lexicase selection, only if the following two conditions are met:
  \begin{enumerate}
    \item $\exists j\in\{1,\cdots,N\}$ for which $f_j(y_i) = f^*_j(\mathcal{Y})$.
    \item $y_i \nprec y_k \ \forall y_k \in \mathcal{Y}$.
  \end{enumerate}
\end{theorem}

We hereby propose an efficient algorithm to find Pareto set boundaries in a population through vectorized pairwise comparisons. This step helps us identify individuals with zero and non-zero probabilities to be selected by lexicase selection.

The basic idea is to check each individual in the population if Theorem~\ref{thm:lexiprob} holds. It is clear that if the $E(y_i)=0$, $y_i$ can not be selected by lexicase selection (due to the second condition). For the pairwise comparisons, since one individual may dominate multiple individuals, and we can remove the individual from the comparison pool once it is dominated (due to the first condition), it is better to start with the best individual and compare it against others. So we start by sorting the individuals by their number of elitism $E(y_i)$, as individuals with high elitism are more likely to dominate others. Another benefit of sorting is to reduce the number of comparisons. As shown in Theorem~\ref{thm:elitism}, if $E(y_1)>E(y_2)$, $y_1 \nprec y_2$. So starting from the first individual in the list sorted by $E(y_i)$ from high to low, we only need to compare the current individual with individuals that have equal or lower $E(y_i)$.

The comparisons are performed in a vectorized fashion. Given the fitness vector $f(y_i) = [f_1(y_i), f_2(y_i),\cdots,f_N(y_i)]^T$, if for any case $j$ we have $f_j(y_1) \geq f_j(y_2)$, then we can infer that $y_1 \succeq y_2$. We use vectorized comparison to check if $\sum{f(y_1) < f(y_2)}=0$. Since the fitness vectors are all distinct after preselection, according to Lemma~\ref{lem:strict}, we have $y_1 \succ y_2$.

During the iterations, if an individual is found to be dominated, we record that and also remove it from the sorted list. This is because dominance follows transitive property, \ie, if $y_i$ is dominated by $y_k$, any of its dominants are also dominated by $y_k$, so removing $y_i$ will avoid unnecessary comparisons.

\subsection{Assigning Probabilities to Individuals}

After getting the Pareto set boundaries, we already know which individuals will and will not be selected by lexicase selection, and we can assign a selection probability of zero to individuals that are not in the Pareto set boundaries. The next step is to determine the selection probabilities for Pareto set boundaries.

It has been proved that calculating the exact probabilities is an NP-Hard problem~\cite{dolson2023calculating}. So here we propose an efficient algorithm to alternatively approximate the probabilities. Our hypothesis for the approximation is: for any training case, suppose there are $k$ individuals being elites in that case, the probability of which individual to be selected is based on their total elitism, \ie, the number of best fitness achieved. The intuition is that when we look at the training first case in the sequence of a lexicase selection event, suppose $k$ individuals tie on this training case, then instead of going to the next case like normal lexicase selection, we assume that whichever case has a higher chance to be the best on the next case is more likely to be finally selected. After computing the distribution for each individual on each training case, we average over all the cases to get the final distribution.

More formally, the unnormalized probability density function for an individual $y_i$ in the population $\mathcal{Y}$ on the $j$th training case is given by
\begin{equation}
  \label{eq:prob1}
  h_j(y_i)= \begin{cases}
  E(y_i) & \text{if } f_j(y_i) = f_j^*(\mathcal{Y}) \\
  0 & \text{otherwise}
  \end{cases}.
\end{equation}

The probability distribution of selecting individuals on the $j$th training case is
\begin{equation}
  \label{eq:prob2}
  P_j(y_i) = \frac{h_j(y_i)}{\sum_{y_k \in \mathcal{Y}} h_j(y_k)}.
\end{equation}

Finally, the probability distribution of selecting individuals is averaged over all the training cases,
\begin{equation}
  \label{eq:prob3}
  P(y_i) = \frac{\sum_{j=1}^{N} P_j(y_i)}{N},
\end{equation}
which is used to assign probabilities in the proposed plexicase selection.

\begin{table*}[t]
  \caption{Results on program synthesis benchmark problems. We report the number of successes over 100 runs. With fixed $\alpha=1$, plexicase selection is competitive to lexicase selection on all the problems. With probability manipulation (selecting the best $\alpha\in\{0.5,1,2\}$), plexicase selection outperforms lexicase selection on some problems. We also perform Pearson's Chi-squared tests to compare the number of successes of plexicase selection against lexicase selection, and the results are marked with significance levels ($P\leq 0.05$*, $0.01$**, $0.001$***).}
  \label{tab:gp}
  \begin{center}
    \begin{tabular}{l|ccc|ccc}
      \toprule
      \multirow{2}{*}{Problem}   & \multicolumn{3}{c|}{Regular} & \multicolumn{3}{c}{Downsampled}                                                                         \\
      \cmidrule(lr){2-4}\cmidrule(lr){5-7}
                                 & Lexicase                     & Plexicase ($\alpha=1$)          & Plexicase    & Lexicase     & Plexicase ($\alpha=1$) & Plexicase      \\
      \midrule
      compare-string-lengths     & 0                            & 0                               & 0            & 0            & 0                      & 0              \\
      median                     & 89                           & 83                              & \textbf{96}  & 70           & 88**                   & \textbf{98}*** \\
      negative-to-zero           & 79                           & 80                              & \textbf{85}  & 76           & 90*                    & \textbf{94}**  \\
      number-io                  & 99                           & \textbf{100}                    & \textbf{100} & 99           & \textbf{100}           & \textbf{100}   \\
      replace-space-with-newline & 10                           & 6                               & \textbf{13}  & 11           & 7                      & \textbf{12}    \\
      smallest                   & 100                          & 100                             & 100          & \textbf{100} & 99                     & \textbf{100}   \\
      vector-average             & \textbf{100}                 & 99                              & \textbf{100} & 100          & 100                    & 100            \\
      \bottomrule
    \end{tabular}
  \end{center}
\end{table*}

\subsection{Probabilistic Lexicase Selection}
\label{sec:pls}

With aforementioned definitions in mind, we introduce probabilistic lexicase selection (plexicase selection). In general, plexicase selection first find the Pareto set boundaries through pairwise comparisons, then assign probabilities to each individual to form the probability distribution of selection. Finally, the parent selection is performed by sampling from the generated distribution. The complete algorithm is outlined in Algorithm~\ref{alg:method}.

\begin{algorithm}[t]
  \KwData{
    \begin{itemize}
      \item $f(y_i)$: the fitness vector of individuals $y_i \in \mathcal{Y}$
      \item $\mathcal{Y}$: the population of individuals
    \end{itemize}
  }
  \KwResult{
    \begin{itemize}
      \item $Pa$: individuals that are selected as parents
    \end{itemize}
  }

  $\mathcal{P} \gets$ Find Pareto Set Boundaries (Algorithm~\ref{alg:pareto})

  \For{$y_i$ in $\mathcal{Y}$}{
    \If{$y_i \in \mathcal{P}$}{
      Calculate $P(y_i)$ (Equation~\ref{eq:prob1},\ref{eq:prob2},\ref{eq:prob3})
    }\Else{$P(y_i) = 0$}
  }

  $Pa \gets$ sample parents from the calculated distribution $P(\mathcal{Y})$

  \KwRet{$Pa$}

  \caption{Probabilistic Lexicase Selection}
  \label{alg:method}
\end{algorithm}

There are two major benefits of our method, which samples parents from a calculated distribution instead of running repeated selection events like the original lexicase selection. First, once the probabilities are calculated, sampling can be efficiently performed to obtain numerous parents from the distribution. The worst-case runtime of our method is the same as running lexicase selection for just one selection event. In practice, we observe that our method is significantly faster than lexicase selection in various tasks. Secondly, having a probabilistic representation of the individuals being selected can provide us with flexibility to further apply parametric operations to improve the performance.

In particular, we propose a probability manipulation strategy to control the kurtosis (tailedness) of the selection distribution, \ie, the randomness in multiple selection events. We introduce a hyperparameter $\alpha \geq 0$, which acts similarly to a temperature parameter in the Softmax function. After calculating the probability of selecting each individual $P(y_i)$, we perform the following probability manipulation:
\begin{equation}
  P'(y_i) = \frac{P(y_i)^\alpha}{\sum_{y_i\in\mathcal{Y}}P(y_i)^\alpha}
\end{equation}
where $P'(y_i)$ is the new probability distribution to be used in plexicase selection. 

Intuitively, $\alpha$ controls the difference in probabilities of each individual to be selected. When $\alpha=1$, there is no change in the distribution. When $\alpha$ is small, the probabilities are more evenly distributed, (\eg, at the extreme case when $\alpha=0$, all the individuals that are Pareto set boundaries will be selected uniformly at random), whereas when $\alpha$ is large, the probabilities are more skewed towards the elites.

Note that such manipulation does not change the probability of individuals with an initial probability of zero, so it keeps the basic terminology of selection that only Pareto set boundaries will be selected, as stated in Theorem~\ref{thm:pareto}.

\subsection{Extension with $\epsilon$-Relaxation}

The proposed plexicase selection can be easily extended to its $\epsilon$-relaxed form, namely $\epsilon$-plexicase selection, to handle tasks where the fitness measure is in the continuous space. In particular, $\epsilon$ refers to the dynamic relaxation on elitism proposed in prior work~\cite{la2016epsilon}. We briefly describe the key modifications needed for the extension to $\epsilon$-plexicase selection.

Firstly, most of the definitions in Section~\ref{sec:pairwise} can be extended to their $\epsilon$-relaxed forms. In general, we have the following definition regarding $\epsilon$-domination:

\begin{definition}
  For individuals $y_1, y_2 \in \mathcal{Y}$, if $f_j(y_1) - \epsilon \geq f_j(y_2) \ \forall j \in \{1,\cdots,N\}$, we state $y_1 \succeq y_2$ ($y_1$ $\epsilon$-dominates $y_2$).
\end{definition}

Similarly, Algorithm~\ref{alg:pareto} can be extended to find $\epsilon$-relaxed Pareto set boundaries if we use $\epsilon$-domination to address elitism and domination. By replacing Algorithm~\ref{alg:pareto} with its $\epsilon$-relaxed form in Algorithm~\ref{alg:method}, we get the $\epsilon$-relaxed form of plexicase selection, \ie, $\epsilon$-plexicase selection.

\section{Experiments}
\label{sec:4}

We conduct experiments to validate the proposed method in two domains: program synthesis and symbolic regression, in which lexicase selection has been demonstrated to be a state-of-the-art approach. In this section, we describe the implementation details and experimental results.

\subsection{Datasets and Benchmarks}

For program synthesis, we use a sample of problems taken from the General Program Synthesis Benchmark Suite~\cite{helmuth2015general}, which is a standard benchmark for testing GP systems for program synthesis. The problems contained in this suite were sourced from introductory computer science textbooks and programming competitions. This aligns the complexity of the problems with the average skill of a beginner human programmer. Solutions to these problems require the use of various data types, control flow, and basic data structures. We have selected the same subset of 7 problems used in \cite{pantridge2020code} out of the full PSB suite in order to compare the change in performance when using plexicase. State-of-the-art solution rates to these problems for PushGP~\cite{helmuth2022problem}, Grammar Guided Genetic Programming~\cite{forstenlechner2018extending}, and Code Building GP~\cite{pantridge2022functional} show that this subset of the PSB covers a range of difficulties regardless of the GP method.

Another domain that receives rising interest from the GP and ML community is symbolic regression. We adopt the recently proposed SRBench~\cite{la2021contemporary}, which is a large-scale, open-source, reproducible benchmarking platform for symbolic regression. In particular, we use a subset of 20 black-box regression problems with sample sizes varying from 40 to 200. These regression problems consist of both real-world problems (data obtained from physical observations) and synthetic problems (data generated from static functions or simulations), covering diverse domains such as health informatics, business, environmental science, and economics.

\subsection{Implementation Details}

For program synthesis, we implement our algorithm in the Code Building Genetic Programming (CBGP) framework~\cite{pantridge2020code}. CBGP is a general program synthesis system that works with arbitrary data types and structures, and shows competitive problem-solving performance. The original CBGP system uses lexicase selection as the default parent selection method, and we compare it with plexicase selection. We follow the original CBGP paper~\cite{pantridge2020code} to set the system configurations: for each problem, we perform $100$ runs of the system with a population size of $1000$ for $300$ generations; the variation operator used in all the runs is UMAD~\cite{helmuth2018program}. For plexicase selection, we use three different $\alpha=0.5,1,2$ for probability manipulation and report the best result for each problem out of the three configurations.

For symbolic regression, we start with the \textit{gplearn}\footnote{https://github.com/trevorstephens/gplearn} framework, which originally uses tournament selection. Since symbolic regression problems need to optimize with continuous values, we implement both $\epsilon$-lexicase selection and $\epsilon$-plexicase selection, and compare them against the tournament selection baseline. For tournament and $\epsilon$-lexicase selection, we adopt the same configurations as in SRBench~\cite{la2021contemporary}, which are $6$ combinations of hyperparameters with $3$ population sizes of $1000,500,100$ (the according numbers of generations are $500,1000,5000$) and $2$ function sets. The tournament size is $20$ for all the problems. For $\epsilon$-plexicase selection, we include different $\alpha$ values ($0.5$, $1$, and $2$) as hyperparameters, but we limit the number of combinations to six for fair comparisons with other methods. Each method is tested for $10$ repeated trials on each of the $20$ problem ($200$ trials in total) with different random seeds controlling both the train/test split and the initialization of the algorithm.

\begin{figure}[t]
  \begin{center}
    \includegraphics[width=.99\linewidth]{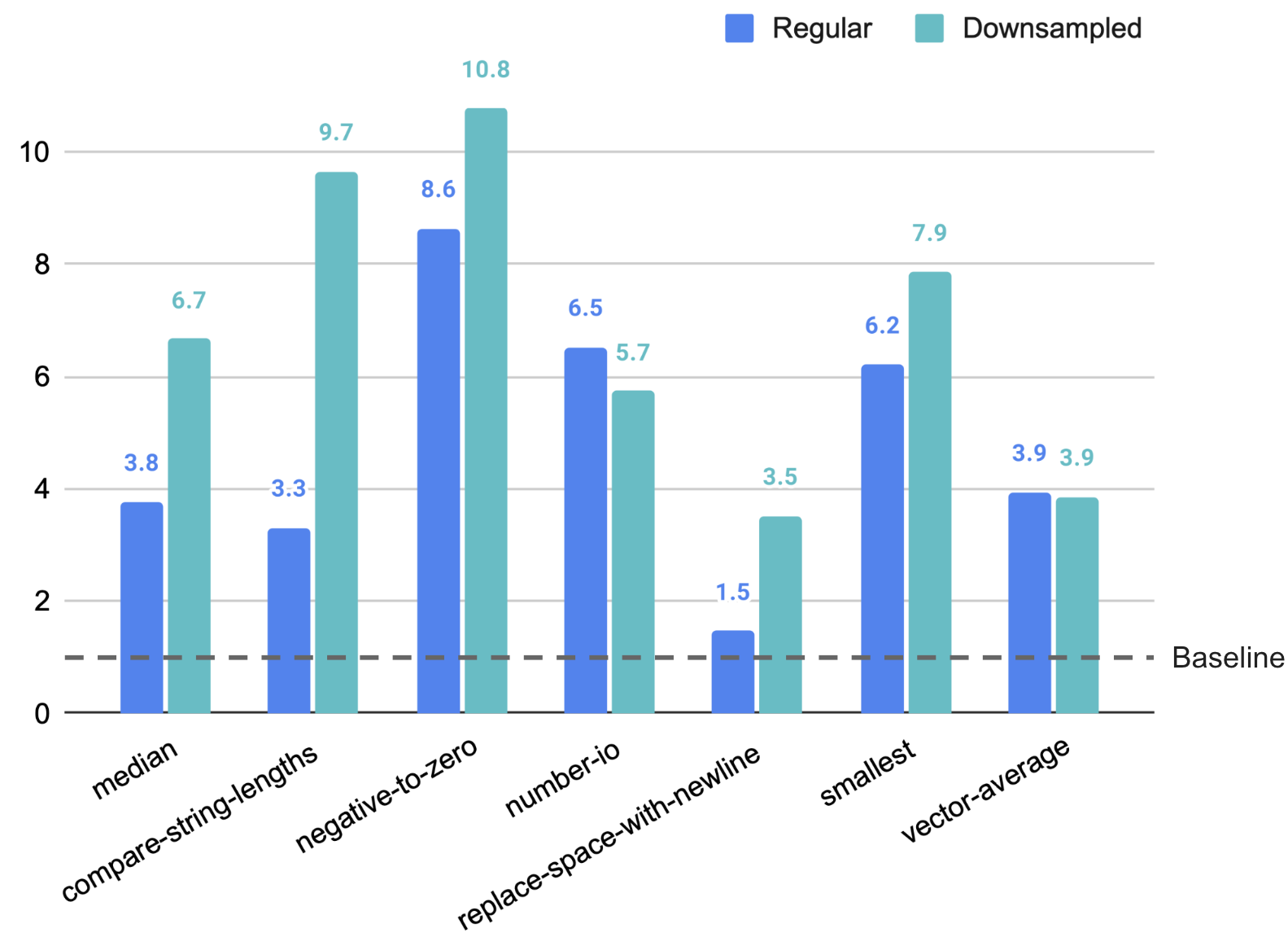}
  \end{center}
  \caption{Average runtime speed-up of plexicase selection compared to lexicase selection in paired comparisons. A speed-up of $n$ means the plexicase selection runtime is $1/n$ of the lexicase selection runtime. The comparisons are conducted in both regular and downsampled circumstances.
  }
  \label{fig:gptime}
\end{figure}

\begin{table}[t]
  \caption{Ablation study on different $\alpha$ values for probability manipulation of plexicase selection on program synthesis benchmark problems. We report the number of successes over 100 runs.}
  \label{tab:gpalpha}
  \begin{center}
    \begin{tabular}{lccc}
      \toprule
      \multirow{2}{*}{Problem}   & \multicolumn{3}{c}{Plexicase (fixed $\alpha$)}                               \\
      \cmidrule(lr){2-4}
                                 & $\alpha=0.5$                                   & $\alpha=1$   & $\alpha=2$   \\
      \midrule
      compare-string-lengths     & 0                                              & 0            & 0            \\
      median                     & 61                                             & 83           & \textbf{96}  \\
      negative-to-zero           & \textbf{85}                                    & 80           & 70           \\
      number-io                  & 99                                             & \textbf{100} & \textbf{100} \\
      replace-space-with-newline & 4                                              & 6            & \textbf{13}  \\
      smallest                   & 98                                             & \textbf{100} & \textbf{100} \\
      vector-average             & \textbf{100}                                   & 99           & \textbf{100} \\
      \bottomrule
    \end{tabular}
  \end{center}
\end{table}

\subsection{Results on Program Synthesis}

\subsubsection{Problem-solving performance}

Table~\ref{tab:gp} shows the problem-solving performance of plexicase selection compared to the regular lexicase selection. We report the number of successes over $100$ runs. With fixed $\alpha=1$, plexicase is competitive to lexicase on all the problems. With probability manipulation (selecting the best $\alpha\in\{0.5,1,2\}$), our method is able to outperform lexicase on some problems. We also perform Pearson's Chi-squared tests~\cite{pearson1900x} to compare plexicase selection against lexicase selection, showing that there is no significant difference in problem-solving performance between regular lexicase and plexicase selection.

Recent work~\cite{hernandez2019random,helmuth2022problem} has demonstrated that the advantages of lexicase selection can be potentially amplified by downsampling. Following this trend, we also test the algorithms with a downsampling rate of $0.25$. The downsampling is performed before each generation to use only a subset of training cases. We observe that plexicase selection benefits more from downsampling, and is able to significantly outperform lexicase selection on $2$ problems with downsampling. Such results indicate that plexicase is likely to work well on larger-scale problems with downsampling.

\begin{figure*}[t]
  \begin{center}
    \includegraphics[width=.95\linewidth]{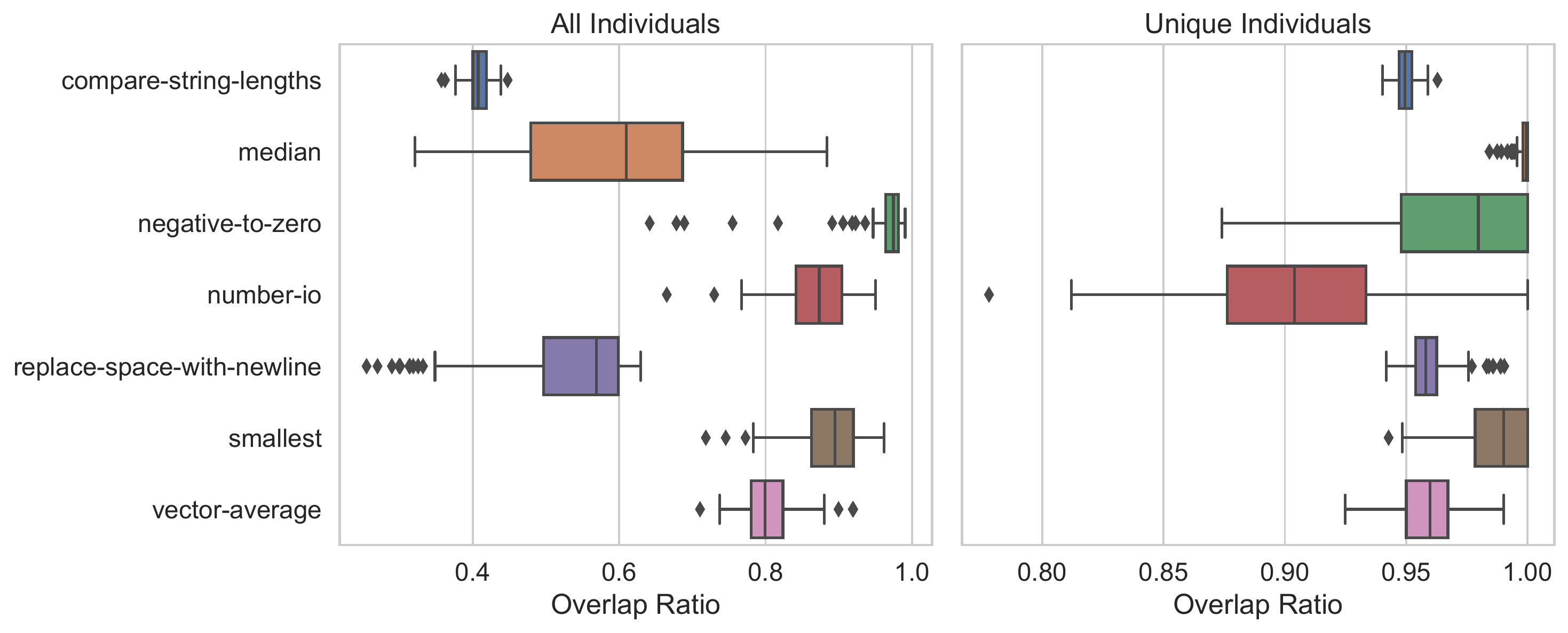}
  \end{center}
  \caption{Overlap between individuals selected by plexicase selection and those selected by lexicase selection. The ratios are calculated for all individuals as well as unique individuals.}
  \label{fig:gpoverlap}
\end{figure*}

\subsubsection{Runtime analysis}

Another important factor we care about is efficiency. In order to fairly compare the runtime of plexicase and lexicase selection, we design a paired-comparison experiment. During each generation, we run both lexicase selection and plexicase selection on the same population with the same training data, and record the runtime just for each selection algorithm.

Figure~\ref{fig:gptime} shows the average runtime speed-up of plexicase selection to lexicase selection in paired comparisons. A speed-up of $n$ means the runtime is $1/n$ of lexicase runtime. We can see that plexicase achieves significant speed-up on all the problems. With downsampling, the speed-up is further enhanced on most of the problems. Such results indicate that our method is more efficient than lexicase selection on the benchmark program synthesis problems, and the advantage is further enhanced with downsampling.

\subsubsection{Effect of probability manipulation}

As an ablation study, we report the detailed results of plexicase selection with different $\alpha$ values for probability manipulation. As shown in Table~\ref{tab:gpalpha}, for some problems, \eg, \texttt{median} and \texttt{replace-space-with-newline}, larger $\alpha$ value gives better results, but for other problems like \texttt{negative-to-zero}, small $\alpha$ works better. These observations indicate that different problems may favor different levels of kurtosis of the individual distribution. In other words, some problems require more exploitation with good solutions, others need more exploration with a large diversity of solutions. While lexicase selection does not have control over the trade-off between exploration and exploitation, our method manages to overcome this disadvantage and demonstrates better performance.

\begin{figure*}[t]
  \begin{center}
    \includegraphics[width=.79\linewidth]{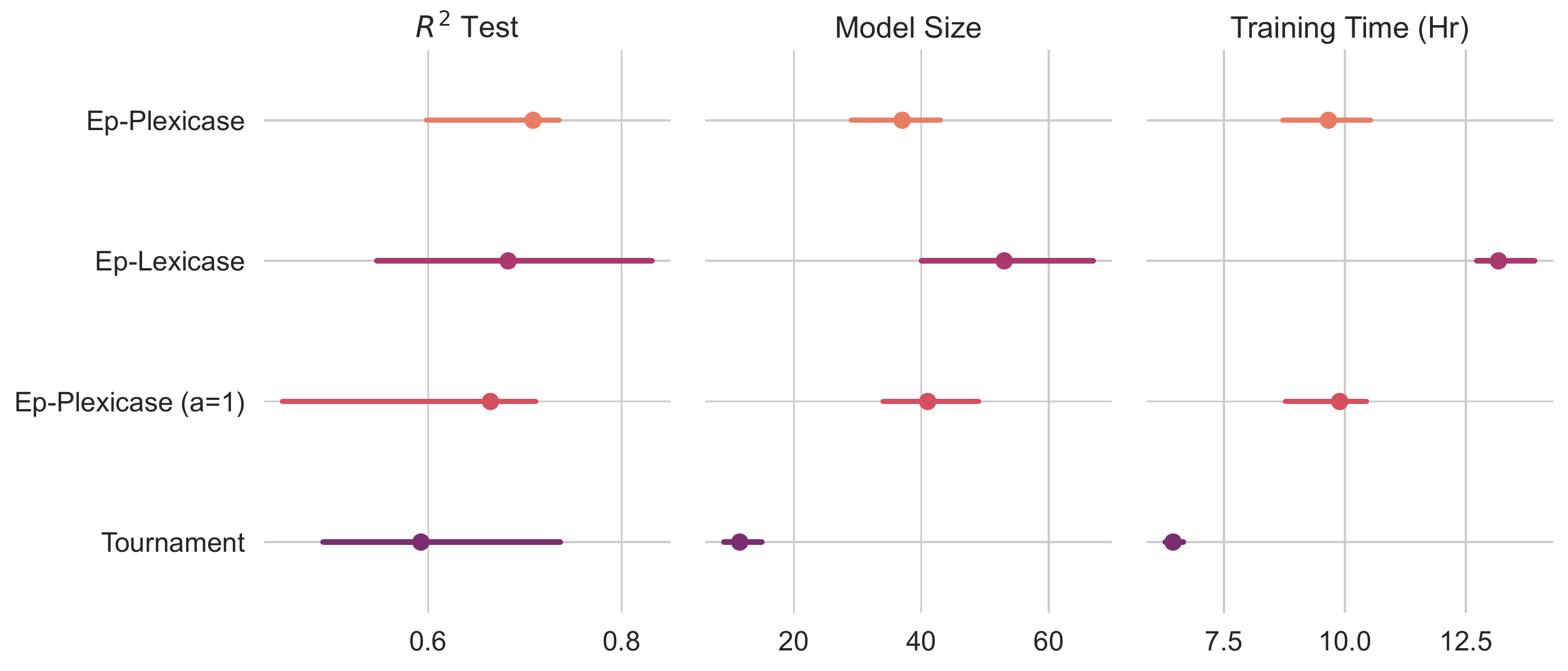}
  \end{center}
  \caption{Testing performance on the black-box benchmark problems. Each algorithm is evaluated for 200 trials in total. Points indicate the median values of each metric among all 20 problems, where for each problem we take the median value over 10 repeated trials. The bars show the 0.95 confidence interval obtained from bootstrapping. 
  }
  \label{fig:sr1}
\end{figure*}

\subsubsection{Selection similarity}

In order to validate how well plexicase selection is approximating the probabilities of selecting individuals with lexicase selection, we perform an ablation study on the similarity of individuals selected by two algorithms. The similarity is measured as the overlap ratio of individuals selected by both algorithms over those selected by lexicase selection. The overlap is separately calculated over all individuals as well as just unique individuals. We report the distribution of mean overlap ratios over generations for all runs, as summarized in Figure~\ref{fig:gpoverlap}. We can see that for $4$ problems, the overlap ratios of all individuals are above $0.8$, indicating that plexicase is approximating the distribution with considerable overlap. There also exists some variances of overlap ratios among different problems, which suggests that the approximation performance may be problem-dependent.

Similar observations have been found in recent work~\cite{helmuth2022population} that lexicase selection produces varying population diversity on different problems. However, the overlap ratios of unique individuals are almost always above $0.9$, indicating that most of the individuals selected by lexicase selection have been selected at least once by plexicase selection. This result validates the theoretical correctness of plexicase selection in obtaining the Pareto set boundaries as individuals that can be selected by lexicase selection.

\subsection{Results on Symbolic Regression}

\subsubsection{Performance on black-box regression problems}

We assess the performance of our method on symbolic regression problems in terms of both accuracy and complexity. For accuracy, we use the coefficient of determination, \ie, $R^2$, which is defined as
\begin{equation}
  R^2=\frac{\sum_i(y_i-\hat{y_i})^2}{\sum_i (y_i-\bar{y})^2}.
\end{equation}
For complexity, we follow SRBench and calculate the number of mathematical operators, features, and constants in the model.

The median performance over all 20 black-box benchmark problems is summarized in Figure~\ref{fig:sr1}. For each problem, we take the median performance over 10 repeated trials with different random seeds, following SRBench's recommended evaluation setting. We can see that using probability manipulation, $\epsilon$-plexicase selection slightly outperforms $\epsilon$-lexicase selection in regression accuracy, with significantly less model size and training time.

We also perform a Mood's median test~\cite{mood1950introduction} comparing the $R^2$ values of $\epsilon$-plexicase and $\epsilon$-lexicase selection. The resulting p-value is $0.8256$, indicating that there is no significant difference in performance between the two algorithms, \ie, $\epsilon$-plexicase well approximates the problem-solving capability of $\epsilon$-lexicase selection. In addition, the performance of $\epsilon$-plexicase selection may be further improved with more budget on hyperparameter tuning to compensate for the reduction in training time.

The results also show that $\epsilon$-plexicase selection has less variance in performance across different problems, compared to $\epsilon$-lexicase selection, indicating that our method is more stable and robust in solving the symbolic regression task in general. It is worth noting that the variances are inevitable since SRBench aggregates many ML problems and the problems are with varying difficulties.

\subsubsection{Runtime analysis}

Figure~\ref{fig:sr1} shows a significant reduction in the overall training time of $\epsilon$-plexicase selection compared to $\epsilon$-lexicase selection. It should be noted that training time includes other processes, such as evaluations of individuals, which is equal for both algorithms. This suggests that the improvement in the runtime of the selection process alone is substantial.

To further investigate the patterns of runtime improvement with $\epsilon$-plexicase selection, we plot the runtime of both methods separately for each problem, characterized by the total number of samples, as shown in Figure~\ref{fig:sr2}. The results indicate that $\epsilon$-plexicase selection consistently outperforms lexicase selection in terms of runtime across all problems, irrespective of the dataset size. This finding highlights the efficiency of our approach in solving the symbolic regression tasks.

\begin{figure}[t]
  \begin{center}
    \includegraphics[width=.99\linewidth]{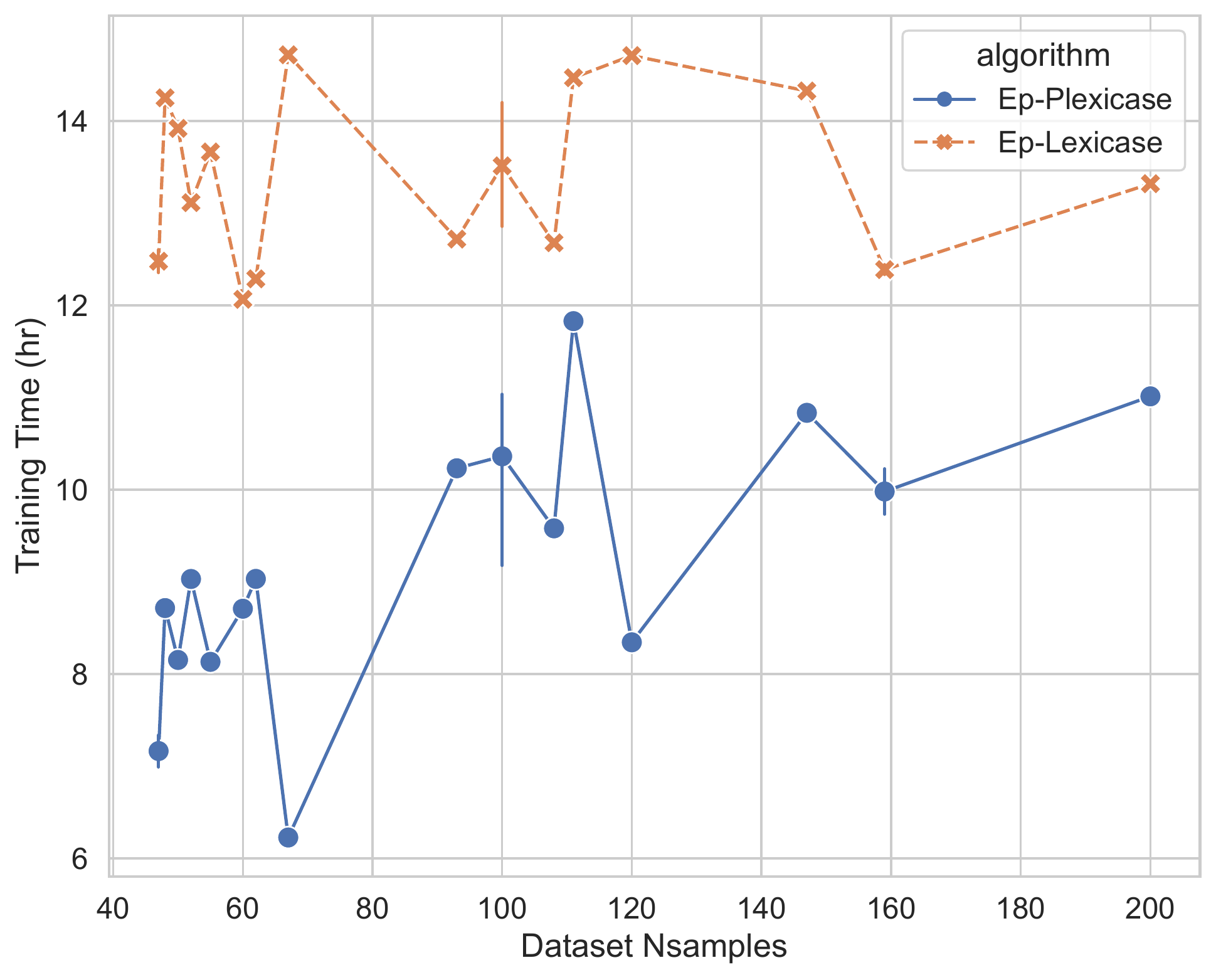}
  \end{center}
  \caption{Runtime comparison between $\epsilon$-plexicase selection and $\epsilon$-lexicase selection on symbolic regression problems.
  }
  \label{fig:sr2}
\end{figure}

\section{Conclusion and Future Work}

In this work, we introduce a novel parent selection method called Probabilistic Lexicase Selection, which efficiently approximates the probability distribution of lexicase selection. The proposed method not only demonstrates superior problem-solving capabilities as a semantic-aware selection method, but also benefits from having a probabilistic representation of selection for enhanced efficiency and flexibility. Specifically, the probability distribution enables us to efficiently sample numerous parents at once instead of performing repeated selection events. We also introduce a probability manipulation method to further enhance the problem-solving performance of plexicase selection by controlling the kurtosis of distribution.

To validate the proposed method, we conducted experiments in two prevalent task domains: program synthesis and symbolic regression. The empirical results on standard benchmarks demonstrate that our plexicase selection algorithm achieves state-of-the-art problem-solving performance comparable to lexicase selection while also surpassing lexicase selection in terms of computation efficiency. Ablation studies further investigate the effect of hyperparameters and validate the correctness of the proposed method.

One limitation of this work is that although we tested our method on two popular domains for genetic programming, the benchmark problems used in the experiments are relatively small-scale compared to modern machine learning tasks. In future work, we plan to extend plexicase selection to solve larger-scale tasks such as evolutionary optimization for deep neural networks~\cite{ding2022optimizing} and parameterized quantum computing models~\cite{ding2022evolutionary,ding2023multi}.

\begin{acks}
  This material is based upon work supported by the National Science Foundation under Grant No. 2117377. Any opinions, findings, and conclusions or recommendations expressed in this publication are those of the authors and do not necessarily reflect the views of the National Science Foundation. This work was performed in part using high performance computing equipment obtained from the Collaborative R\&D Fund managed by the Massachusetts Technology Collaborative. The authors would like to thank Thomas Helmuth, Ryan Boldi, and Anil Saini for their helpful comments and suggestions.
\end{acks}

\bibliographystyle{ACM-Reference-Format}
\bibliography{main.bib}


\begin{thebibliography}{48}


\ifx \showCODEN    \undefined \def \showCODEN     #1{\unskip}     \fi
\ifx \showDOI      \undefined \def \showDOI       #1{#1}\fi
\ifx \showISBNx    \undefined \def \showISBNx     #1{\unskip}     \fi
\ifx \showISBNxiii \undefined \def \showISBNxiii  #1{\unskip}     \fi
\ifx \showISSN     \undefined \def \showISSN      #1{\unskip}     \fi
\ifx \showLCCN     \undefined \def \showLCCN      #1{\unskip}     \fi
\ifx \shownote     \undefined \def \shownote      #1{#1}          \fi
\ifx \showarticletitle \undefined \def \showarticletitle #1{#1}   \fi
\ifx \showURL      \undefined \def \showURL       {\relax}        \fi
\providecommand\bibfield[2]{#2}
\providecommand\bibinfo[2]{#2}
\providecommand\natexlab[1]{#1}
\providecommand\showeprint[2][]{arXiv:#2}

\bibitem[\protect\citeauthoryear{Aenugu and Spector}{Aenugu and
  Spector}{2019}]%
        {aenugu2019lexicase}
\bibfield{author}{\bibinfo{person}{Sneha Aenugu} {and} \bibinfo{person}{Lee
  Spector}.} \bibinfo{year}{2019}\natexlab{}.
\newblock \showarticletitle{Lexicase selection in learning classifier systems}.
  In \bibinfo{booktitle}{\emph{Proceedings of the Genetic and Evolutionary
  Computation Conference}}. \bibinfo{pages}{356--364}.
\newblock


\bibitem[\protect\citeauthoryear{Baker}{Baker}{2014}]%
        {baker2014adaptive}
\bibfield{author}{\bibinfo{person}{James~Edward Baker}.}
  \bibinfo{year}{2014}\natexlab{}.
\newblock \showarticletitle{Adaptive selection methods for genetic algorithms}.
  In \bibinfo{booktitle}{\emph{Proceedings of the first international
  conference on genetic algorithms and their applications}}. Psychology Press,
  \bibinfo{pages}{101--106}.
\newblock


\bibitem[\protect\citeauthoryear{Baker et~al\mbox{.}}{Baker
  et~al\mbox{.}}{1987}]%
        {baker1987reducing}
\bibfield{author}{\bibinfo{person}{James~E Baker} {et~al\mbox{.}}}
  \bibinfo{year}{1987}\natexlab{}.
\newblock \showarticletitle{Reducing bias and inefficiency in the selection
  algorithm}. In \bibinfo{booktitle}{\emph{Proceedings of the second
  international conference on genetic algorithms}}, Vol.~\bibinfo{volume}{206}.
  \bibinfo{pages}{14--21}.
\newblock


\bibitem[\protect\citeauthoryear{Boldi, Briesch, Sobania, Lalejini, Helmuth,
  Rothlauf, Ofria, and Spector}{Boldi et~al\mbox{.}}{2023}]%
        {boldi2023informed}
\bibfield{author}{\bibinfo{person}{Ryan Boldi}, \bibinfo{person}{Martin
  Briesch}, \bibinfo{person}{Dominik Sobania}, \bibinfo{person}{Alexander
  Lalejini}, \bibinfo{person}{Thomas Helmuth}, \bibinfo{person}{Franz
  Rothlauf}, \bibinfo{person}{Charles Ofria}, {and} \bibinfo{person}{Lee
  Spector}.} \bibinfo{year}{2023}\natexlab{}.
\newblock \showarticletitle{Informed Down-Sampled Lexicase Selection:
  Identifying productive training cases for efficient problem solving}.
\newblock \bibinfo{journal}{\emph{arXiv preprint arXiv:2301.01488}}
  (\bibinfo{year}{2023}).
\newblock


\bibitem[\protect\citeauthoryear{De~Melo, Vargas, and Banzhaf}{De~Melo
  et~al\mbox{.}}{2019}]%
        {de2019batch}
\bibfield{author}{\bibinfo{person}{Vin{\'\i}cius~V De~Melo},
  \bibinfo{person}{Danilo~Vasconcellos Vargas}, {and} \bibinfo{person}{Wolfgang
  Banzhaf}.} \bibinfo{year}{2019}\natexlab{}.
\newblock \showarticletitle{Batch tournament selection for genetic programming:
  the quality of lexicase, the speed of tournament}. In
  \bibinfo{booktitle}{\emph{Proceedings of the genetic and evolutionary
  computation conference}}. \bibinfo{pages}{994--1002}.
\newblock


\bibitem[\protect\citeauthoryear{Ding, Boldi, Helmuth, and Spector}{Ding
  et~al\mbox{.}}{2022a}]%
        {ding2022going}
\bibfield{author}{\bibinfo{person}{Li Ding}, \bibinfo{person}{Ryan Boldi},
  \bibinfo{person}{Thomas Helmuth}, {and} \bibinfo{person}{Lee Spector}.}
  \bibinfo{year}{2022}\natexlab{a}.
\newblock \showarticletitle{Going faster and hence further with lexicase
  selection}. In \bibinfo{booktitle}{\emph{Proceedings of the Genetic and
  Evolutionary Computation Conference Companion}}. \bibinfo{pages}{538--541}.
\newblock


\bibitem[\protect\citeauthoryear{Ding, Boldi, Helmuth, and Spector}{Ding
  et~al\mbox{.}}{2022b}]%
        {ding2022lexicase}
\bibfield{author}{\bibinfo{person}{Li Ding}, \bibinfo{person}{Ryan Boldi},
  \bibinfo{person}{Thomas Helmuth}, {and} \bibinfo{person}{Lee Spector}.}
  \bibinfo{year}{2022}\natexlab{b}.
\newblock \showarticletitle{Lexicase selection at scale}. In
  \bibinfo{booktitle}{\emph{Proceedings of the Genetic and Evolutionary
  Computation Conference Companion}}. \bibinfo{pages}{2054--2062}.
\newblock


\bibitem[\protect\citeauthoryear{Ding and Spector}{Ding and Spector}{2021}]%
        {ding2021evolving}
\bibfield{author}{\bibinfo{person}{Li Ding} {and} \bibinfo{person}{Lee
  Spector}.} \bibinfo{year}{2021}\natexlab{}.
\newblock \showarticletitle{Evolving neural selection with adaptive
  regularization}. In \bibinfo{booktitle}{\emph{Proceedings of the Genetic and
  Evolutionary Computation Conference Companion}}. \bibinfo{pages}{1717--1725}.
\newblock


\bibitem[\protect\citeauthoryear{Ding and Spector}{Ding and Spector}{2022a}]%
        {ding2022evolutionary}
\bibfield{author}{\bibinfo{person}{Li Ding} {and} \bibinfo{person}{Lee
  Spector}.} \bibinfo{year}{2022}\natexlab{a}.
\newblock \showarticletitle{Evolutionary quantum architecture search for
  parametrized quantum circuits}. In \bibinfo{booktitle}{\emph{Proceedings of
  the Genetic and Evolutionary Computation Conference Companion}}.
  \bibinfo{pages}{2190--2195}.
\newblock


\bibitem[\protect\citeauthoryear{Ding and Spector}{Ding and Spector}{2022b}]%
        {ding2022optimizing}
\bibfield{author}{\bibinfo{person}{Li Ding} {and} \bibinfo{person}{Lee
  Spector}.} \bibinfo{year}{2022}\natexlab{b}.
\newblock \showarticletitle{Optimizing Neural Networks with Gradient Lexicase
  Selection}. In \bibinfo{booktitle}{\emph{International Conference on Learning
  Representations}}.
\newblock


\bibitem[\protect\citeauthoryear{Ding and Spector}{Ding and Spector}{2023}]%
        {ding2023multi}
\bibfield{author}{\bibinfo{person}{Li Ding} {and} \bibinfo{person}{Lee
  Spector}.} \bibinfo{year}{2023}\natexlab{}.
\newblock \showarticletitle{Multi-Objective Evolutionary Architecture Search
  for Parameterized Quantum Circuits}.
\newblock \bibinfo{journal}{\emph{Entropy}} \bibinfo{volume}{25},
  \bibinfo{number}{1} (\bibinfo{year}{2023}), \bibinfo{pages}{93}.
\newblock


\bibitem[\protect\citeauthoryear{Dolson}{Dolson}{2023}]%
        {dolson2023calculating}
\bibfield{author}{\bibinfo{person}{Emily Dolson}.}
  \bibinfo{year}{2023}\natexlab{}.
\newblock \showarticletitle{Calculating lexicase selection probabilities is
  NP-Hard}.
\newblock \bibinfo{journal}{\emph{arXiv preprint arXiv:2301.06724}}
  (\bibinfo{year}{2023}).
\newblock


\bibitem[\protect\citeauthoryear{Ferguson, Hernandez, Junghans, Lalejini,
  Dolson, and Ofria}{Ferguson et~al\mbox{.}}{2019}]%
        {ferguson2019characterizing}
\bibfield{author}{\bibinfo{person}{Austin~J. Ferguson},
  \bibinfo{person}{Jose~Guadalupe Hernandez}, \bibinfo{person}{Daniel
  Junghans}, \bibinfo{person}{Alexander Lalejini}, \bibinfo{person}{Emily
  Dolson}, {and} \bibinfo{person}{Charles Ofria}.}
  \bibinfo{year}{2019}\natexlab{}.
\newblock \showarticletitle{Characterizing the effects of random subsampling
  and dilution on Lexicase selection}, In \bibinfo{booktitle}{Genetic
  Programming Theory and Practice XVII},
  \bibfield{editor}{\bibinfo{person}{Wolfgang Banzhaf}, \bibinfo{person}{Erik
  Goodman}, \bibinfo{person}{Leigh Sheneman}, \bibinfo{person}{Leonardo
  Trujillo}, {and} \bibinfo{person}{Bill Worzel}} (Eds.).
\newblock \bibinfo{journal}{\emph{Genetic Programming Theory and Practice
  XVII}}.
\newblock


\bibitem[\protect\citeauthoryear{Fieldsend and Moraglio}{Fieldsend and
  Moraglio}{2015}]%
        {fieldsend2015strength}
\bibfield{author}{\bibinfo{person}{Jonathan~E Fieldsend} {and}
  \bibinfo{person}{Alberto Moraglio}.} \bibinfo{year}{2015}\natexlab{}.
\newblock \showarticletitle{Strength through diversity: Disaggregation and
  multi-objectivisation approaches for genetic programming}. In
  \bibinfo{booktitle}{\emph{Proceedings of the 2015 Annual Conference on
  Genetic and Evolutionary Computation}}. \bibinfo{pages}{1031--1038}.
\newblock


\bibitem[\protect\citeauthoryear{Forstenlechner, Fagan, Nicolau, and
  O'Neill}{Forstenlechner et~al\mbox{.}}{2018}]%
        {forstenlechner2018extending}
\bibfield{author}{\bibinfo{person}{Stefan Forstenlechner},
  \bibinfo{person}{David Fagan}, \bibinfo{person}{Miguel Nicolau}, {and}
  \bibinfo{person}{Michael O'Neill}.} \bibinfo{year}{2018}\natexlab{}.
\newblock \showarticletitle{Extending program synthesis grammars for
  grammar-guided genetic programming}. In \bibinfo{booktitle}{\emph{Parallel
  Problem Solving from Nature--PPSN XV: 15th International Conference, Coimbra,
  Portugal, September 8--12, 2018, Proceedings, Part I 15}}. Springer,
  \bibinfo{pages}{197--208}.
\newblock


\bibitem[\protect\citeauthoryear{Galvan-Lopez, Cody-Kenny, Trujillo, and
  Kattan}{Galvan-Lopez et~al\mbox{.}}{2013}]%
        {galvan2013using}
\bibfield{author}{\bibinfo{person}{Edgar Galvan-Lopez},
  \bibinfo{person}{Brendan Cody-Kenny}, \bibinfo{person}{Leonardo Trujillo},
  {and} \bibinfo{person}{Ahmed Kattan}.} \bibinfo{year}{2013}\natexlab{}.
\newblock \showarticletitle{Using semantics in the selection mechanism in
  genetic programming: a simple method for promoting semantic diversity}. In
  \bibinfo{booktitle}{\emph{2013 IEEE Congress on Evolutionary Computation}}.
  IEEE, \bibinfo{pages}{2972--2979}.
\newblock


\bibitem[\protect\citeauthoryear{Golberg}{Golberg}{1989}]%
        {golberg1989genetic}
\bibfield{author}{\bibinfo{person}{David~E Golberg}.}
  \bibinfo{year}{1989}\natexlab{}.
\newblock \showarticletitle{Genetic algorithms in search, optimization, and
  machine learning}.
\newblock \bibinfo{journal}{\emph{Addion wesley}} \bibinfo{volume}{1989},
  \bibinfo{number}{102} (\bibinfo{year}{1989}), \bibinfo{pages}{36}.
\newblock


\bibitem[\protect\citeauthoryear{Goldberg and Deb}{Goldberg and Deb}{1991}]%
        {goldberg1991comparative}
\bibfield{author}{\bibinfo{person}{David~E Goldberg} {and}
  \bibinfo{person}{Kalyanmoy Deb}.} \bibinfo{year}{1991}\natexlab{}.
\newblock \showarticletitle{A comparative analysis of selection schemes used in
  genetic algorithms}.
\newblock In \bibinfo{booktitle}{\emph{Foundations of genetic algorithms}}.
  Vol.~\bibinfo{volume}{1}. \bibinfo{publisher}{Elsevier},
  \bibinfo{pages}{69--93}.
\newblock


\bibitem[\protect\citeauthoryear{Helmuth, Lengler, and La~Cava}{Helmuth
  et~al\mbox{.}}{2022}]%
        {helmuth2022population}
\bibfield{author}{\bibinfo{person}{Thomas Helmuth}, \bibinfo{person}{Johannes
  Lengler}, {and} \bibinfo{person}{William La~Cava}.}
  \bibinfo{year}{2022}\natexlab{}.
\newblock \showarticletitle{Population Diversity Leads to Short Running Times
  of Lexicase Selection}. In \bibinfo{booktitle}{\emph{Parallel Problem Solving
  from Nature--PPSN XVII: 17th International Conference, PPSN 2022, Dortmund,
  Germany, September 10--14, 2022, Proceedings, Part II}}. Springer,
  \bibinfo{pages}{485--498}.
\newblock


\bibitem[\protect\citeauthoryear{Helmuth, McPhee, and Spector}{Helmuth
  et~al\mbox{.}}{2016a}]%
        {helmuth2016impact}
\bibfield{author}{\bibinfo{person}{Thomas Helmuth},
  \bibinfo{person}{Nicholas~Freitag McPhee}, {and} \bibinfo{person}{Lee
  Spector}.} \bibinfo{year}{2016}\natexlab{a}.
\newblock \showarticletitle{The impact of hyperselection on lexicase
  selection}. In \bibinfo{booktitle}{\emph{Proceedings of the Genetic and
  Evolutionary Computation Conference 2016}}. \bibinfo{pages}{717--724}.
\newblock


\bibitem[\protect\citeauthoryear{Helmuth, McPhee, and Spector}{Helmuth
  et~al\mbox{.}}{2016b}]%
        {helmuth2016lexicase}
\bibfield{author}{\bibinfo{person}{Thomas Helmuth},
  \bibinfo{person}{Nicholas~Freitag McPhee}, {and} \bibinfo{person}{Lee
  Spector}.} \bibinfo{year}{2016}\natexlab{b}.
\newblock \showarticletitle{Lexicase selection for program synthesis: a
  diversity analysis}.
\newblock In \bibinfo{booktitle}{\emph{Genetic Programming Theory and Practice
  XIII}}. \bibinfo{publisher}{Springer}, \bibinfo{pages}{151--167}.
\newblock


\bibitem[\protect\citeauthoryear{Helmuth, McPhee, and Spector}{Helmuth
  et~al\mbox{.}}{2018}]%
        {helmuth2018program}
\bibfield{author}{\bibinfo{person}{Thomas Helmuth},
  \bibinfo{person}{Nicholas~Freitag McPhee}, {and} \bibinfo{person}{Lee
  Spector}.} \bibinfo{year}{2018}\natexlab{}.
\newblock \showarticletitle{Program synthesis using uniform mutation by
  addition and deletion}. In \bibinfo{booktitle}{\emph{Proceedings of the
  Genetic and Evolutionary Computation Conference}}.
  \bibinfo{pages}{1127--1134}.
\newblock


\bibitem[\protect\citeauthoryear{Helmuth, Pantridge, and Spector}{Helmuth
  et~al\mbox{.}}{2019}]%
        {helmuth2019lexicase}
\bibfield{author}{\bibinfo{person}{Thomas Helmuth}, \bibinfo{person}{Edward
  Pantridge}, {and} \bibinfo{person}{Lee Spector}.}
  \bibinfo{year}{2019}\natexlab{}.
\newblock \showarticletitle{Lexicase Selection of Specialists}. In
  \bibinfo{booktitle}{\emph{Proceedings of the Genetic and Evolutionary
  Computation Conference}} (Prague, Czech Republic)
  \emph{(\bibinfo{series}{GECCO '19})}. \bibinfo{publisher}{Association for
  Computing Machinery}, \bibinfo{address}{New York, NY, USA},
  \bibinfo{pages}{1030--1038}.
\newblock
\showISBNx{9781450361118}
\urldef\tempurl%
\url{https://doi.org/10.1145/3321707.3321875}
\showDOI{\tempurl}


\bibitem[\protect\citeauthoryear{Helmuth, Pantridge, and Spector}{Helmuth
  et~al\mbox{.}}{2020}]%
        {helmuth2020importance}
\bibfield{author}{\bibinfo{person}{Thomas Helmuth}, \bibinfo{person}{Edward
  Pantridge}, {and} \bibinfo{person}{Lee Spector}.}
  \bibinfo{year}{2020}\natexlab{}.
\newblock \showarticletitle{On the Importance of Specialists for Lexicase
  Selection}.
\newblock \bibinfo{journal}{\emph{Genetic Programming and Evolvable Machines}}
  \bibinfo{volume}{21}, \bibinfo{number}{3} (\bibinfo{date}{sep}
  \bibinfo{year}{2020}), \bibinfo{pages}{349--373}.
\newblock
\showISSN{1389-2576}
\urldef\tempurl%
\url{https://doi.org/10.1007/s10710-020-09377-2}
\showDOI{\tempurl}


\bibitem[\protect\citeauthoryear{Helmuth and Spector}{Helmuth and
  Spector}{2015}]%
        {helmuth2015general}
\bibfield{author}{\bibinfo{person}{Thomas Helmuth} {and} \bibinfo{person}{Lee
  Spector}.} \bibinfo{year}{2015}\natexlab{}.
\newblock \showarticletitle{General program synthesis benchmark suite}. In
  \bibinfo{booktitle}{\emph{Proceedings of the 2015 Annual Conference on
  Genetic and Evolutionary Computation}}. \bibinfo{pages}{1039--1046}.
\newblock


\bibitem[\protect\citeauthoryear{Helmuth and Spector}{Helmuth and
  Spector}{2020}]%
        {helmuth2020explaining}
\bibfield{author}{\bibinfo{person}{Thomas Helmuth} {and} \bibinfo{person}{Lee
  Spector}.} \bibinfo{year}{2020}\natexlab{}.
\newblock \showarticletitle{Explaining and exploiting the advantages of
  down-sampled lexicase selection}. In \bibinfo{booktitle}{\emph{ALIFE 2020:
  The 2020 Conference on Artificial Life}}. MIT Press,
  \bibinfo{pages}{341--349}.
\newblock


\bibitem[\protect\citeauthoryear{Helmuth and Spector}{Helmuth and
  Spector}{2022}]%
        {helmuth2022problem}
\bibfield{author}{\bibinfo{person}{Thomas Helmuth} {and} \bibinfo{person}{Lee
  Spector}.} \bibinfo{year}{2022}\natexlab{}.
\newblock \showarticletitle{Problem-solving benefits of down-sampled lexicase
  selection}.
\newblock \bibinfo{journal}{\emph{Artificial Life}} \bibinfo{volume}{27},
  \bibinfo{number}{3--4} (\bibinfo{year}{2022}), \bibinfo{pages}{183--203}.
\newblock


\bibitem[\protect\citeauthoryear{Helmuth, Spector, and Matheson}{Helmuth
  et~al\mbox{.}}{2014}]%
        {helmuth2014solving}
\bibfield{author}{\bibinfo{person}{Thomas Helmuth}, \bibinfo{person}{Lee
  Spector}, {and} \bibinfo{person}{James Matheson}.}
  \bibinfo{year}{2014}\natexlab{}.
\newblock \showarticletitle{Solving uncompromising problems with lexicase
  selection}.
\newblock \bibinfo{journal}{\emph{IEEE Transactions on Evolutionary
  Computation}} \bibinfo{volume}{19}, \bibinfo{number}{5}
  (\bibinfo{year}{2014}), \bibinfo{pages}{630--643}.
\newblock


\bibitem[\protect\citeauthoryear{Hernandez, Lalejini, Dolson, and
  Ofria}{Hernandez et~al\mbox{.}}{2019}]%
        {hernandez2019random}
\bibfield{author}{\bibinfo{person}{Jose~Guadalupe Hernandez},
  \bibinfo{person}{Alexander Lalejini}, \bibinfo{person}{Emily Dolson}, {and}
  \bibinfo{person}{Charles Ofria}.} \bibinfo{year}{2019}\natexlab{}.
\newblock \showarticletitle{Random subsampling improves performance in lexicase
  selection}. In \bibinfo{booktitle}{\emph{Proceedings of the Genetic and
  Evolutionary Computation Conference Companion}}. \bibinfo{pages}{2028--2031}.
\newblock


\bibitem[\protect\citeauthoryear{Huizinga and Clune}{Huizinga and
  Clune}{2018}]%
        {huizinga2018evolving}
\bibfield{author}{\bibinfo{person}{Joost Huizinga} {and} \bibinfo{person}{Jeff
  Clune}.} \bibinfo{year}{2018}\natexlab{}.
\newblock \showarticletitle{Evolving multimodal robot behavior via many
  stepping stones with the combinatorial multi-objective evolutionary
  algorithm}.
\newblock \bibinfo{journal}{\emph{arXiv preprint arXiv:1807.03392}}
  (\bibinfo{year}{2018}).
\newblock


\bibitem[\protect\citeauthoryear{Krawiec and Liskowski}{Krawiec and
  Liskowski}{2015}]%
        {krawiec2015automatic}
\bibfield{author}{\bibinfo{person}{Krzysztof Krawiec} {and}
  \bibinfo{person}{Pawe{\l} Liskowski}.} \bibinfo{year}{2015}\natexlab{}.
\newblock \showarticletitle{Automatic derivation of search objectives for
  test-based genetic programming}. In \bibinfo{booktitle}{\emph{European
  Conference on Genetic Programming}}. Springer, \bibinfo{pages}{53--65}.
\newblock


\bibitem[\protect\citeauthoryear{La~Cava, Helmuth, Spector, and Moore}{La~Cava
  et~al\mbox{.}}{2019}]%
        {la2019probabilistic}
\bibfield{author}{\bibinfo{person}{William La~Cava}, \bibinfo{person}{Thomas
  Helmuth}, \bibinfo{person}{Lee Spector}, {and} \bibinfo{person}{Jason~H
  Moore}.} \bibinfo{year}{2019}\natexlab{}.
\newblock \showarticletitle{A probabilistic and multi-objective analysis of
  lexicase selection and $\varepsilon$-lexicase selection}.
\newblock \bibinfo{journal}{\emph{Evolutionary Computation}}
  \bibinfo{volume}{27}, \bibinfo{number}{3} (\bibinfo{year}{2019}),
  \bibinfo{pages}{377--402}.
\newblock


\bibitem[\protect\citeauthoryear{La~Cava and Moore}{La~Cava and Moore}{2018}]%
        {la2018behavioral}
\bibfield{author}{\bibinfo{person}{William La~Cava} {and}
  \bibinfo{person}{Jason Moore}.} \bibinfo{year}{2018}\natexlab{}.
\newblock \showarticletitle{Behavioral search drivers and the role of elitism
  in soft robotics}. In \bibinfo{booktitle}{\emph{ALIFE 2018: The 2018
  Conference on Artificial Life}}. MIT Press, \bibinfo{pages}{206--213}.
\newblock


\bibitem[\protect\citeauthoryear{La~Cava and Moore}{La~Cava and Moore}{2020a}]%
        {la2020genetic}
\bibfield{author}{\bibinfo{person}{William La~Cava} {and}
  \bibinfo{person}{Jason~H Moore}.} \bibinfo{year}{2020}\natexlab{a}.
\newblock \showarticletitle{Genetic programming approaches to learning fair
  classifiers}. In \bibinfo{booktitle}{\emph{Proceedings of the 2020 Genetic
  and Evolutionary Computation Conference}}. \bibinfo{pages}{967--975}.
\newblock


\bibitem[\protect\citeauthoryear{La~Cava and Moore}{La~Cava and Moore}{2020b}]%
        {la2020learning}
\bibfield{author}{\bibinfo{person}{William La~Cava} {and}
  \bibinfo{person}{Jason~H Moore}.} \bibinfo{year}{2020}\natexlab{b}.
\newblock \showarticletitle{Learning feature spaces for regression with genetic
  programming}.
\newblock \bibinfo{journal}{\emph{Genetic Programming and Evolvable Machines}}
  \bibinfo{volume}{21}, \bibinfo{number}{3} (\bibinfo{year}{2020}),
  \bibinfo{pages}{433--467}.
\newblock


\bibitem[\protect\citeauthoryear{La~Cava, Orzechowski, Burlacu,
  de~Fran{\c{c}}a, Virgolin, Jin, Kommenda, and Moore}{La~Cava
  et~al\mbox{.}}{2021}]%
        {la2021contemporary}
\bibfield{author}{\bibinfo{person}{William La~Cava}, \bibinfo{person}{Patryk
  Orzechowski}, \bibinfo{person}{Bogdan Burlacu},
  \bibinfo{person}{Fabr{\'\i}cio~Olivetti de Fran{\c{c}}a},
  \bibinfo{person}{Marco Virgolin}, \bibinfo{person}{Ying Jin},
  \bibinfo{person}{Michael Kommenda}, {and} \bibinfo{person}{Jason~H Moore}.}
  \bibinfo{year}{2021}\natexlab{}.
\newblock \showarticletitle{Contemporary symbolic regression methods and their
  relative performance}. In \bibinfo{booktitle}{\emph{Proceedings of the Neural
  Information Processing Systems Track on Datasets and Benchmarks}},
  Vol.~\bibinfo{volume}{1}.
\newblock


\bibitem[\protect\citeauthoryear{La~Cava, Singh, Taggart, Suri, and
  Moore}{La~Cava et~al\mbox{.}}{2018}]%
        {la2018learning}
\bibfield{author}{\bibinfo{person}{William La~Cava}, \bibinfo{person}{Tilak~Raj
  Singh}, \bibinfo{person}{James Taggart}, \bibinfo{person}{Srinivas Suri},
  {and} \bibinfo{person}{Jason~H Moore}.} \bibinfo{year}{2018}\natexlab{}.
\newblock \showarticletitle{Learning concise representations for regression by
  evolving networks of trees}.
\newblock \bibinfo{journal}{\emph{arXiv preprint arXiv:1807.00981}}
  (\bibinfo{year}{2018}).
\newblock


\bibitem[\protect\citeauthoryear{La~Cava, Spector, and Danai}{La~Cava
  et~al\mbox{.}}{2016}]%
        {la2016epsilon}
\bibfield{author}{\bibinfo{person}{William La~Cava}, \bibinfo{person}{Lee
  Spector}, {and} \bibinfo{person}{Kourosh Danai}.}
  \bibinfo{year}{2016}\natexlab{}.
\newblock \showarticletitle{Epsilon-lexicase selection for regression}. In
  \bibinfo{booktitle}{\emph{Proceedings of the Genetic and Evolutionary
  Computation Conference 2016}}. \bibinfo{pages}{741--748}.
\newblock


\bibitem[\protect\citeauthoryear{Lipowski and Lipowska}{Lipowski and
  Lipowska}{2012}]%
        {lipowski2012roulette}
\bibfield{author}{\bibinfo{person}{Adam Lipowski} {and} \bibinfo{person}{Dorota
  Lipowska}.} \bibinfo{year}{2012}\natexlab{}.
\newblock \showarticletitle{Roulette-wheel selection via stochastic
  acceptance}.
\newblock \bibinfo{journal}{\emph{Physica A: Statistical Mechanics and its
  Applications}} \bibinfo{volume}{391}, \bibinfo{number}{6}
  (\bibinfo{year}{2012}), \bibinfo{pages}{2193--2196}.
\newblock


\bibitem[\protect\citeauthoryear{Liskowski, Krawiec, Helmuth, and
  Spector}{Liskowski et~al\mbox{.}}{2015}]%
        {liskowski2015comparison}
\bibfield{author}{\bibinfo{person}{Pawel Liskowski}, \bibinfo{person}{Krzysztof
  Krawiec}, \bibinfo{person}{Thomas Helmuth}, {and} \bibinfo{person}{Lee
  Spector}.} \bibinfo{year}{2015}\natexlab{}.
\newblock \showarticletitle{Comparison of semantic-aware selection methods in
  genetic programming}. In \bibinfo{booktitle}{\emph{Proceedings of the
  Companion Publication of the 2015 Annual Conference on Genetic and
  Evolutionary Computation}}. \bibinfo{pages}{1301--1307}.
\newblock


\bibitem[\protect\citeauthoryear{Mood}{Mood}{1950}]%
        {mood1950introduction}
\bibfield{author}{\bibinfo{person}{Alexander~McFarlane Mood}.}
  \bibinfo{year}{1950}\natexlab{}.
\newblock \showarticletitle{Introduction to the Theory of Statistics.}
\newblock  (\bibinfo{year}{1950}).
\newblock


\bibitem[\protect\citeauthoryear{Moore and Stanton}{Moore and Stanton}{2017}]%
        {moore2017lexicase}
\bibfield{author}{\bibinfo{person}{Jared~M Moore} {and} \bibinfo{person}{Adam
  Stanton}.} \bibinfo{year}{2017}\natexlab{}.
\newblock \showarticletitle{Lexicase selection outperforms previous strategies
  for incremental evolution of virtual creature controllers}. In
  \bibinfo{booktitle}{\emph{ECAL 2017, the Fourteenth European Conference on
  Artificial Life}}. MIT Press, \bibinfo{pages}{290--297}.
\newblock


\bibitem[\protect\citeauthoryear{Moore and Stanton}{Moore and Stanton}{2018}]%
        {moore2018tiebreaks}
\bibfield{author}{\bibinfo{person}{Jared~M Moore} {and} \bibinfo{person}{Adam
  Stanton}.} \bibinfo{year}{2018}\natexlab{}.
\newblock \bibinfo{title}{{Tiebreaks and Diversity: Isolating Effects in
  Lexicase Selection}}.
\newblock , \bibinfo{numpages}{590--597}~pages.
\newblock
\urldef\tempurl%
\url{https://doi.org/10.1162/isal_a_00109}
\showDOI{\tempurl}


\bibitem[\protect\citeauthoryear{Pantridge, Helmuth, and Spector}{Pantridge
  et~al\mbox{.}}{2022}]%
        {pantridge2022functional}
\bibfield{author}{\bibinfo{person}{Edward Pantridge}, \bibinfo{person}{Thomas
  Helmuth}, {and} \bibinfo{person}{Lee Spector}.}
  \bibinfo{year}{2022}\natexlab{}.
\newblock \showarticletitle{Functional code building genetic programming}. In
  \bibinfo{booktitle}{\emph{Proceedings of the Genetic and Evolutionary
  Computation Conference}}. \bibinfo{pages}{1000--1008}.
\newblock


\bibitem[\protect\citeauthoryear{Pantridge and Spector}{Pantridge and
  Spector}{2020}]%
        {pantridge2020code}
\bibfield{author}{\bibinfo{person}{Edward Pantridge} {and} \bibinfo{person}{Lee
  Spector}.} \bibinfo{year}{2020}\natexlab{}.
\newblock \showarticletitle{Code building genetic programming}. In
  \bibinfo{booktitle}{\emph{Proceedings of the 2020 Genetic and Evolutionary
  Computation Conference}}. \bibinfo{pages}{994--1002}.
\newblock


\bibitem[\protect\citeauthoryear{Pearson}{Pearson}{1900}]%
        {pearson1900x}
\bibfield{author}{\bibinfo{person}{Karl Pearson}.}
  \bibinfo{year}{1900}\natexlab{}.
\newblock \showarticletitle{X. On the criterion that a given system of
  deviations from the probable in the case of a correlated system of variables
  is such that it can be reasonably supposed to have arisen from random
  sampling}.
\newblock \bibinfo{journal}{\emph{The London, Edinburgh, and Dublin
  Philosophical Magazine and Journal of Science}} \bibinfo{volume}{50},
  \bibinfo{number}{302} (\bibinfo{year}{1900}), \bibinfo{pages}{157--175}.
\newblock


\bibitem[\protect\citeauthoryear{Spector}{Spector}{2012}]%
        {spector2012assessment}
\bibfield{author}{\bibinfo{person}{Lee Spector}.}
  \bibinfo{year}{2012}\natexlab{}.
\newblock \showarticletitle{Assessment of problem modality by differential
  performance of lexicase selection in genetic programming: a preliminary
  report}. In \bibinfo{booktitle}{\emph{Proceedings of the 14th annual
  conference companion on Genetic and evolutionary computation}}.
  \bibinfo{pages}{401--408}.
\newblock


\bibitem[\protect\citeauthoryear{Troise and Helmuth}{Troise and
  Helmuth}{2018}]%
        {troise2018lexicase}
\bibfield{author}{\bibinfo{person}{Sarah~Anne Troise} {and}
  \bibinfo{person}{Thomas Helmuth}.} \bibinfo{year}{2018}\natexlab{}.
\newblock \showarticletitle{Lexicase selection with weighted shuffle}.
\newblock In \bibinfo{booktitle}{\emph{Genetic Programming Theory and Practice
  XV}}. \bibinfo{publisher}{Springer}, \bibinfo{pages}{89--104}.
\newblock


\end{thebibliography}

\newpage
\appendix

\section{Lexicase Selection}
\label{asec:lex}

Lexicase selection~\citep{helmuth2014solving,spector2012assessment} is a parent selection method in population-based evolutionary algorithms. An example algorithm of applying lexicase selection in a common genetic programming setting is outlined in Alg.~\ref{alg:lexi}.

The main concept behind lexicase selection is that, during each selection event, a randomly shuffled sequence of training cases is considered. The individuals that survive are only those whose fitness is the optimal among all the cases that have been considered so far, according to a specific ordering. The selection process continues as the sequence of cases is traversed, until only one candidate is left or until all the training cases have been examined. If there are still remaining candidates at the end, a random selection is made from among them.

As the ordering of training cases is randomized for each selection event in lexicase selection, every training case has a chance to be prioritized and placed at the beginning of the sequence. This allows for the possibility of selecting specialist individuals that may have lower average performance but excel in one or more specific cases compared to other individuals.

\begin{algorithm}
  \KwData{
    \begin{itemize}
      \item \texttt{cases} - randomly shuffled sequence of data samples
      \item \texttt{candidates} - the entire population of programs
    \end{itemize}
  }
  \KwResult{
    \begin{itemize}
      \item an individual program to be used as a parent
    \end{itemize}
  }

  \For{\texttt{case} in \texttt{cases}}{
    \texttt{candidates} $\gets$ the subset of the current \texttt{candidates} that have exactly best performance on \texttt{case}

    \If{\texttt{candidates} contains only one single \texttt{candidate}}{\KwRet{\texttt{candidate}}}
  }

  \KwRet{a randomly selected \texttt{candidate} in \texttt{candidates}}

  \caption{Lexicase selection to select one parent in genetic programming~\cite{ding2022optimizing}}
  \label{alg:lexi}
\end{algorithm}

\section{Supplementary Materials}

The source code to reproduce our experiments can be found at \url{https://github.com/ld-ing/plexicase}.

\end{document}